%% file: main.tex
\documentclass[12pt]{l4dc2021} 
\usepackage{enumerate}
\usepackage{algorithm,algorithmicx,algpseudocode}

\newcommand{\bE}{\mathbb{E}}

\newcommand{\bP}{\mathbb{P}}

\newcommand{\cG}{\mathcal{G}}
\newcommand{\te}{\theta}
\newcommand{\cK}{\mathcal{K}}
\newcommand{\cC}{\mathcal{C}}
\newcommand{\cE}{\mathcal{E}}
\newcommand{\ust}{^{\star}}

\newtheorem{assumption}{\textbf{Assumption}}
\title[]{Reward Biased Maximum Likelihood Estimation for Reinforcement Learning}
\usepackage{times}
\usepackage{xcolor}

\author{%
	\Name{Akshay Mete}\thanks{denotes equal contribution} \Email{akshaymete@tamu.edu}\\
	\addr Texas A\&M University, TX, USA.
	\AND
	\Name{Rahul Singh}\footnotemark[1]\Email{rahulsingh@iisc.ac.in}\\
	\addr Indian Institute of Science, Bangalore, India. 
	\AND
	\Name{Xi Liu} \Email{xiliu.tamu@gmail.com}\\
	\addr Texas A\&M University, TX, USA.
	\AND
	\Name{P. R. Kumar} \Email{prk@tamu.edu}\\
	\addr Texas A\&M University, TX, USA.
}

\begin{document}
	
	\maketitle
	
\begin{abstract}
The Reward-Biased Maximum Likelihood Estimate (RBMLE) for adaptive control
of Markov chains was
proposed in~\citep{kumar_becker_82} to overcome the central
obstacle of what is variously called the fundamental
``closed-identifiability problem" of adaptive control \citep{borkar_varaiya_79},
the ``dual control problem" by Feldbaum \citep{feldbaum1960dual1,feldbaum1960dual2},
or, contemporaneously, the ``exploration vs.~exploitation problem".
It exploited the key observation that since the
maximum likelihood parameter estimator can asymptotically identify the closed-transition
probabilities under a certainty equivalent approach \citep{borkar_varaiya_79}, the limiting parameter
estimates must necessarily have an optimal reward that is less than the
optimal reward attainable for the true but unknown system. Hence it proposed a counteracting reverse bias in
favor of parameters with larger optimal rewards, providing a carefully structured solution to
the fundamental problem alluded to above.
It thereby proposed an optimistic approach of favoring parameters with larger optimal rewards,
now known as ``optimism in the face of uncertainty."
The RBMLE approach has been proved to be long-term average reward optimal in a variety of contexts including
controlled Markov chains, linear quadratic Gaussian (LQG) systems, some nonlinear systems, and diffusions.
However, modern attention is focused on the much finer notion of ``regret," or finite-time performance for \emph{all} time,
espoused by \citep{lai_85}.
Recent analysis of RBMLE for multi-armed stochastic bandits \citep{liu_20} and linear contextual bandits~\citep{hung_20}
has shown that it not only has state-of-the-art regret, but it also exhibits empirical performance comparable to or better than the best
current contenders, and leads to several new and strikingly simple index policies for these classical problems.
Motivated by this, we examine the finite-time performance of RBMLE for
reinforcement learning tasks that involve the general problem of optimal control of unknown Markov Decision Processes. 
We show that it has a regret of $O(\log T)$ over a time horizon of $T$ steps, similar to state-of-the-art algorithms. Simulation studies show that RBMLE outperforms other algorithms such as UCRL2~\citep{auer_09} and Thompson Sampling~\citep{ouyang_17,gopalan_15,abbasi_15}.
\end{abstract}
	
	\begin{keywords}%
		Reinforcement Learning; Markov Decision Process; Adaptive Control%
	\end{keywords}
	
\section{Introduction} \label{Introduction}

Consider a controlled Markov chain with finite state space $X$, finite action set $U$, and controlled transition
probabilities $\bP(x(t+1)=y~|x(t)=x, u(t)=u) = p(x,y,u)$, where $x(t) \in X$ denotes the state at time $t$, and $u(t) \in U$ denotes the action
taken at time $t$. A reward $r(x,u)$ is received when action $u$ is taken in state $x$.
Let $J\ust(p)$ denote the maximal long-term average reward $\liminf_{T \to \infty} \frac{1}{T} \sum_{t=1}^{T} r(x(t),u(t))$
obtainable.
 We consider the case where the transition probabilities $p$
are only known to belong a set $\Theta$, but otherwise unknown.
 We address the adaptive control problem of minimizing the expected ``regret" 
 \begin{align}
 TJ\ust(p) - \bE \sum_{t=1}^{T} r(x(t),u(t) \label{regret definition}
 \end{align}
as a function of $T$.

This broad problem has a long history.
Let $J(\theta,\pi)$ denote the long-term average reward reward accrued by a stationary deterministic policy $\pi : X \to U$
when the transition probabilities are given by $\theta = \{ \theta(x,y,u), x \in X, y \in X, u \in U \}$, let
$J\ust( \theta ) := \max_\pi J(\theta,\pi)$ denote
the optimal long-term average reward attainable under $\theta$, and
let $\pi^\theta \in \arg\max  J(\theta,\pi)$ be an optimal policy for $\theta$. 
In early work, \citep{mandl_74} studied the problem of using a ``certainty equivalent" approach, where a maximum likelihood estimate
(MLE) 
\begin{align}
\hat{\theta}(t) \in \arg\max_{\theta} \prod_{s=1}^{t-1} \theta(x(s), x(s+1),u(s))
\end{align}
 of the unknown transition probabilities is made at each time $t$, and
an action 
$
u(t) = \pi^{\hat{\theta}(t)}(x(t))
$
is taken that is optimal in state $x(t)$ for the transition probabilities $\hat{\theta}(t)$.
Mandl showed that if an ``identifiability condition",
\begin{align}
\theta \neq \theta' \text{ with } \theta, \theta' \in \Theta \implies \theta(x, \cdot , u) \neq \theta' (x, \cdot, u) \quad \forall (x,u) \in X \times U
\end{align}
holds,
then the maximum likelihood estimates $\hat{\theta}(t)$ converge to the true transition probabilities $p$
as $t \to \infty$, and the corresponding long-term average reward obtained by the adaptive controller is 
the optimal reward $J\ust(p)$.
This identifiability condition is however restrictive, e.g., it is not satisfied for the two-armed bandit problem
or any problem with a fundamental exploration vs.~exploitation dilemma.

In general, in the absence of the identifiability condition, \citep{borkar_varaiya_79}
showed that one
only obtains ``closed-loop identification": the maximum likelihood estimates converge to a  $\theta^*$ for which
\begin{align}
\theta^*(x,y, \pi^{\theta^*}(x)) = p(x,y, \pi^{\theta^*}(x)) \quad \quad \forall (x,y). \label{closed-loop identifiability}
\end{align}
However, the limiting policy $\pi^{\theta^*}$ is generally not an optimal long-term average policy for the true transition probabilities $p$.
Indeed this is the central challenge of the exploration vs.~exploitation problem: As the parameter estimates begin to
converge exploration ceases, and one ends up only identifying the behavior of the system under the limited actions being applied to the
system. One misses out on other potentially valuable policies.

This central difficulty was overcome in \citep{kumar_becker_82}. They first noted that (\ref{closed-loop identifiability}) implies that $J(\theta^*,\pi^{\theta^*}) = 
J(p,\pi^{\theta^*})$. As a consequence of this, since $\pi^{\theta^*}$ is optimal for $\theta^*$, i.e., $J(\theta^*,\pi^{\theta^*}) =  J\ust(\theta^*)$,
but not for $p$, i.e., $J(p,\pi^{\theta^*}) \leq J\ust(p)$, they made the critical observation that the
optimal long-term average reward accruable for the limiting estimate must necessarily be lower than the
optimal long-term average reward accruable for the true parameter:
\begin{align}
J\ust(\theta^*) \leq J\ust(p). \label{bias}
\end{align}
Therefore, the maximum likelihood estimator is inherently \emph{biased} in favor of $\theta$'s with lower optimal rewards than $p$.
Therefore to extricate oneself from this bind, one must necessarily tilt the balance toward exploring
parameters with larger optimal rewards.
Motivated by this, they proposed a certainty equivalent approach using a \emph{Reward Biased MLE (RBMLE)}
that attempts to counteract this with a bias in the reverse direction, favoring parameters $\theta$ with a larger optimal reward:
\begin{align}
\hat{\theta}(t) \in \arg\max_{\theta} f(J\ust(\theta))^{\alpha(t)} \prod_{s=1}^{t-1} \theta(x(s), x(s+1),u(s)), \label{RBMLE}
\end{align}
where $f$ is any strictly monotone increasing function.
This biasing however has to be delicate in that $\alpha(t)$ has to be large enough so that
it asymptotically does choose parameters with larger optimal reward than under $p$,
but has to be small enough in that it does not lose the consistency property (\ref{closed-loop identifiability}) of the MLE.
They showed that the choice  $\alpha (t) = o(t)$ with $\lim_{t \to + \infty} \alpha(t) = + \infty$ suffices in ensuring
(\ref{closed-loop identifiability}) for every Cesaro-limit point $\theta^*$ of the RBMLE estimator (\ref{RBMLE}),
but also satisfies
\begin{align}
J\ust(\theta^*) \geq J\ust(p). \label{unbias}
\end{align}
From this it follows that
\begin{align}
J\ust(p) \leq J\ust(\theta^*) = J(\theta^*,\pi^{\theta^*}) =  J(p,\pi^{\theta^*}) \leq J\ust(p), \label{opt}
\end{align}
resulting in $\pi^{\theta^*}$ being an optimal long-term average reward policy for $p$.

%

The RBMLE policy therefore proposed the optimistic philosophy of favoring parameters
with larger rewards, now known as ``optimism in the face of uncertainty'' (OFU).
The inequality (\ref{bias}) indicates why this is fundamentally necessary, since otherwise there is a one-sided exploration bias.
RBMLE was the first long-term
average reward optimal (also called ``asymptotically optimal") learning algorithm in the frequentist setting~\citep{berry1985bandit}
that does not resort to forced explorations.
In the special case of Bernoulli bandits it was shown to yield particularly simple index policies \citep{becker_81}. 
The RBMLE approach has since
been applied to a wide range of sequential decision-making, learning, and adaptive control problems.
The approach was extended to more general MDPs in \citep{kumar_82,kumar_lin_82,borkar1990kumar},
to LQG systems in \citep{kumar1983optimal,campi_98,prandini_00},
to linear time-invariant systems in \citep{bittanti2006adaptive},
to adaptive control of nonlinear systems in \citep{kumar1983simultaneous},
to  more general ergodic probems in \citep{stettner1993nearly},
and to controlled diffusions \citep{borkar1991self,duncan1994almost},
where its long-term average optimality was established.

A finer notion of optimality than long-term average reward optimality is ``regret" (\ref{regret definition}), which was proposed in
\citep{lai_85} in the context of multi-armed bandits (MABs).
Long-term average optimality corresponds to a regret of $o(T)$, but \citep{lai_85} asked the much more delicate question
of how small exactly can regret be made.
They were able to sharply characterize the optimal regret as $c \log T +o(\log T)$ for MABs.
The performance criterion of regret has now become central to the broader field of Reinforcement learning (RL)~\citep{sutton1998introduction}, which
 involves an agent repeatedly interacting with an unknown environment that is modeled as a Markov decision process (MDP)~\citep{puterman2014markov}
 to maximize a total reward. 
Many algorithms such as UCRL \citep{auer_07}, UCRL2 \citep{auer_09}, R-Max \citep{brafman_02}, REGAL \citep{bartlett_12}, Posterior Sampling \citep{strens_00}, \citep{osband_13} and TSMDP \citep{gopalan_15} have been studied in great detail, and their learning regret analyzed.

Lai and Robbins also proposed an ``Upper Confidence Bound" (UCB) policy which plays the bandit whose upper confidence bound is highest,
and showed that it
attains the optimal order of regret.
The UCB policy also employs the OFU principle by trying arms with larger potential rewards, but in a different way
from RBMLE.
It has been extended to a wide variety of learning problems:~\citep{brafman_02,auer_02,auer_09,bartlett_12,singh2020learning}.

While the original work analyzed its long-term average optimality,
the finite-time regret analysis of RBMLE based algorithms
in various settings
is an overdue topic of topical interest.
An initial step in this direction was taken in \citep{liu_20} by analyzing RBMLE for the 
special case of stochastic multi-armed bandits. 
The index policy for Bernoulli bandits suggested in \citep{becker_81}
was generalized to the exponential family of bandits.
They analyzed RBMLE's performance for the exponential family of multi-armed bandits (MABs) and showed that the regret scales as $O(\log T)$.
Moreover, numerical experiments in \citep{liu_20} clearly exhibited that RBMLE outperforms the UCB in terms of empirical regret,
and in fact RBMLE is competitive or slightly better than current state-of-art contenders.
Moreover, RBMLE does so with low computational cost in view of its simple indices.
Recently~\citep{hung_20} have proposed an extension of RBMLE for linear contextual bandits which achieves an $O(\sqrt{T}\log T)$ regret, better than the existing state of the art policies like LinTS~\citep{agrawa_13} and GPUCB~\citep{srinivas_09}. They also show that RBMLE has a competitive regret performance in simulations, and is computationally efficient in comparison with current state-of the art policies such as in~\citep{agrawa_13},~\citep{srinivas_09}.

Another recent effort \citep{abbasi2011regret}, motivated by the RBMLE approach of \citep{campi_98,bittanti2006adaptive},
addressed the performance of regret for linear quadratic Gaussian (LQG) systems, and established a regret of $\tilde{O}(\sqrt{T})$\footnote{$\tilde{O}$ hides factors that are logarithmic in $T$.}.

Due to these developments showing optimal regret performance of RBMLE in these two contexts,
it is of interest to examine the regret performance of RBMLE in more general settings.
This paper takes the first step in finite-time regret analysis and empirical analysis of the RBMLE algorithm for 
reinforcement learning (RL) tasks that involve the general problem of optimal control of unknown Markov Decision Processes.
Its key contributions are:
\begin{enumerate}
\item We propose a new RL algorithm for maximizing rewards for unknown MDPs, that utilizes the RBMLE principle while making control decisions.
\item We analyze the finite-time performance, i.e., the learning regret, of the proposed learning algorithm. We show that the regret is $O (\log T)$.
\item We provide simulation results to show that RBMLE outperforms UCRL2 and TSDE. 
\end{enumerate}
With these results, together with the positive results in the context of stochastic MABs \citep{liu_20} and linear contextual bandits \citep{hung_20}, 
RBMLE provides a second tool for reinforcement learning, complementing the UCB approach.
\section{System Model}
We consider the MDP described in Section~\ref{Introduction}, assuming,
without loss of generality, that $r(x,u) \in [0,1]$ for all $ (x,u) \in X \times U$.
We denote by $\Pi_{sd}$ the set of all stationary deterministic policies that map $X$ into $U$,
by $\Pi_{s}$ the set of all stationary possibly randomized policies,
and by $\Pi\ust(\te) := \arg \max_{\pi \in \Pi_{sd}} J(\te,\pi)$ the set of all optimal stationary deterministic policies for the parameter $\theta$.

\begin{definition}(Unichain MDP)
	Under a stationary policy $\pi$, let
	$\tau^{\pi}_{x,y}$ denote the time taken to hit the state $y$ when started in state $x$. The MDP is called unichain if $\bE [\tau^{\pi}_{x,y}]$ is finite for all $(x,y,\pi)$. 
\end{definition}

\begin{definition}(Mixing Time)\label{def:cond}
	For a unichain MDP with parameter $p$, its mixing time $T_{p}$ is defined as 
	\begin{align*}
		T_p := \max_{\pi \in \Pi_{s}}  \max_{x,y\in X} \bE [\tau^{\pi}_{x,y}].
	\end{align*}
	Its ``conductivity'' $\kappa_p$ is 
	\begin{align*}
		\kappa_p := \max_{\pi \in \Pi_s}  \max_{x\in X}\frac{\max_{y\neq x} \bE [\tau^{\pi}_{x,y}]  }{2\bE [\tau^{\pi}_{x,x}] }.     
	\end{align*}
\end{definition}

\begin{definition}(Gap)
	For a stationary policy $\pi$, let $\Delta(\pi)$ denote the difference between the optimal average reward and the average reward accrued by $\pi$ under parameter $p$,  and by $\Delta_{\min}$ the gap between the
	rewards of the best and second best policies,
	\begin{align*}
		\Delta(\pi) := J\ust(p) - J(p,\pi) \text{ and }  \Delta_{\min}:=\min_{\pi\notin \Pi\ust(p)}\Delta(\pi). 
	\end{align*}
\end{definition}
\begin{definition}(Kullback-Leibler divergence) For $p_2=\{p_2(x)\}$ absolutely continuous with respect to $p_2=\{p_2(x)\}$ the KL-divergence between them is	
\begin{align}
		KL(p_1,p_2):=\sum_{x \in X} p_1(x) \log \frac{p_1(x)}{p_2(x)}.
	\end{align}    
\end{definition}
For two integers $x,y$, we use $[x,y]$ to denote the set $\{x,x+1,\ldots,y\}$ and for $a,b\in\mathbb{R}$ we let $a\vee b := \max\{a,b\}$.

\begin{assumption}\label{assum:1}
We assume that the following information is known about the unknown transition probabilities $p(x,y,u)$ :
\begin{itemize}
\item the set of tuples $(x,y,u)$ for which $p(x,y,u)=0$, 
\item a lower bound $p_{\min}$ on the non-zero transition probabilities,
\begin{align}\label{def:p_min}
p_{\min} := \min\limits_{(x,y,u): p(x,y,u)>0} p(x,y,u).
\end{align}
\end{itemize} 
\rm We let $\Theta$ denote the set 
\begin{align}
\Bigg\{ &\te \in [0,1]^{|X| \times |X| \times |U|}: \te(x,y,u)=0\mbox{ if  }~p(x,y,u)=0, \sum_{y \in X} \te(x,y,u) =1
\;\forall\; (x,u), \te(x,y,u)\ge 0 \Bigg\}. \label{def:Theta}
\end{align}
We occasionally refer to $\te \in \Theta$ as a ``parameter" describing the model or transition probabilites.
\end{assumption}

\section{The RBMLE-Based Learning Algorithm}
For an MDP parameter $\te\in \Theta$, denote by $\te(x,u)$ the vector $\left\{\te(x,y,u)\right\}_{y\in X}$.
Let $n(x,u;t)$ be the number of times an action $u$ has been applied in state $x$ until time $t$, and by $n(x,y,u;t)$ the number of $x\to y$ one-step transitions under the application of action $u$. Let $\hat{p}(t)=\left\{ \hat{p}(x,y,u;t)  \right\}$ be the empirical estimate of $p$ at time $t$, with $\hat{p}(x,y,u;t)$ the MLE of $p(x,y,u)$ at time $t$,
\begin{align}\label{def:empirical}
	\hat{p}(x,y,u;t) : =\frac{n(x,y,u;t)}{n(x,u;t)\vee 1} \;\forall\; x,y \in X \text{ and } u \in U .
\end{align} 
{\bf{The RBMLE algorithm:}} 
The algorithm evolves in an episodic manner. 
For episode $k$, we let $\tau_k$ denote its start time and $\cE_k :=[\tau_k,\tau_{k+1} -1] $ the set of time-slots that comprise it. The episode durations increase exponentially with episode index, with $|\cE_k|=2^k$. Clearly $\tau_k = \sum_{\ell = 1}^{k-1}|\cE_\ell|$.
Throughout we abbreviate $\hat{p}(\tau_k)$ as $\hat{p}_{k}$, $n(x,y;\tau_k)$ as  $n_{k}(x,y)$ and $n(x,y,u;\tau_k)$ as  $n_{k}(x,y,u)$.
At the beginning of each episode $\cE_k$, the RBMLE determines\footnote{Throughout the paper, a pre-specified priority order is used to choose a particular maximizer in $arg\max$ if needed.}:
\begin{enumerate}[(i)]
\item  A ``reward-biased MLE'' $\te_k$:
\begin{align}
&\te_k \in \arg\max_{\te \in \Theta} \bigg\{ \max_{\pi \in \Pi_{sd}}\bigg\{ \alpha(\tau_k) J(\theta,\pi) -\sum_{(x,u)} n_k(x,u)KL\left( \hat{p}_{k}(x,u),\te(x,u)  \right) \bigg\} \bigg\}, \label{eq:rbmle_2}    \\
&\text{where }\alpha(t) :=a\log\left( t^b |X|^{2} |U|\right), \text{ with } b>2, \text{ and } 
a>\frac{|X|^3|U|}{2p_{\min} \Delta_{\min}}. \label{eq:alpha}
\end{align}

\item A stationary deterministic policy $\pi_k \in \arg\max_{\pi \in \Pi_{sd} } J(\te_k,\pi)$ that is optimal for $\te_k$.
\item The action applied for $t \in \cE_k$ is $u(t) = \pi_k (x(t))$.
\end{enumerate}
{\bf{An equivalent Index description of the RBMLE learning algorithm:}}
At the beginning of each episode $\cE_k$, RBMLE attaches an index $I_k(\pi)$ to each $\pi\in\Pi_{sd}$,
\begin{align}\label{eq:index}    
I_k(\pi):=\max_{\te \in \Theta}\bigg\{\alpha(\tau_k) J(\theta,\pi) -\sum_{(x,u)} n_k(x,u)KL\left( \hat{p}_{k}(x,u),\te(x,u)  \right) \bigg\}.
\end{align}
Within $\cE_k$ it implements the policy $\pi_k$ that has the largest index, i.e.,
\begin{align}
	\pi_k \in \arg\max_{\pi \in \Pi_{sd}} I_k(\pi). 
\end{align}
For each $ \pi\in \Pi_{sd}$, define $\theta_{k,\pi}$ as  
\begin{align}\label{eq:index_theta}    
	\theta_{k,\pi} \in \arg \max_{\te \in \Theta}\bigg\{\alpha(\tau_k) J(\theta,\pi) -\sum_{(x,u)} n_k(x,u)KL\left( \hat{p}_{k}(x,u),\te(x,u)  \right) \bigg\}.
\end{align}

%
%
\input{subsecs/prelim}
\input{subsecs/regret}
%
%
\section{Simulation Experiments}
We evaluate the performance of the RBMLE algorithm by empirical comparison with UCRL2 \citep{auer_09} and Thompson Sampling. Among the many Thomspon Sampling variants, we use TSDE \citep{ouyang_17} since the simulation results in \citep{ouyang_17} show that TSDE has lower empirical regret than Lazy PSRL \citep{abbasi_15} and TSMDP \citep{gopalan_15}. We maintain the fairness of the comparison by following approach: For every state-action pair, we generate a sample path of transitions at the beginning of each experiment. We use these same samples for all three algorithms. For the sake of uniformity with UCRL2 and TSDE, the length of an episode of RBMLE is dynamically determined as follows: An episode is terminated if the number of visits to any state-action pair in the episode exceeds the total visits till the beginning of the episode. For all experiments, the bias term for RBMLE 
is $10\log t$ and the confidence parameter for UCRL2 is $0.01$. We compare the cumulative regret for three different MDPs in Figure \ref{fig:1}. In all experiments, RBMLE outperforms UCRL2 and TSDE.
\begin{figure}[t]
\centering
	
	\subfigure{
		\includegraphics[width=0.32\textwidth]{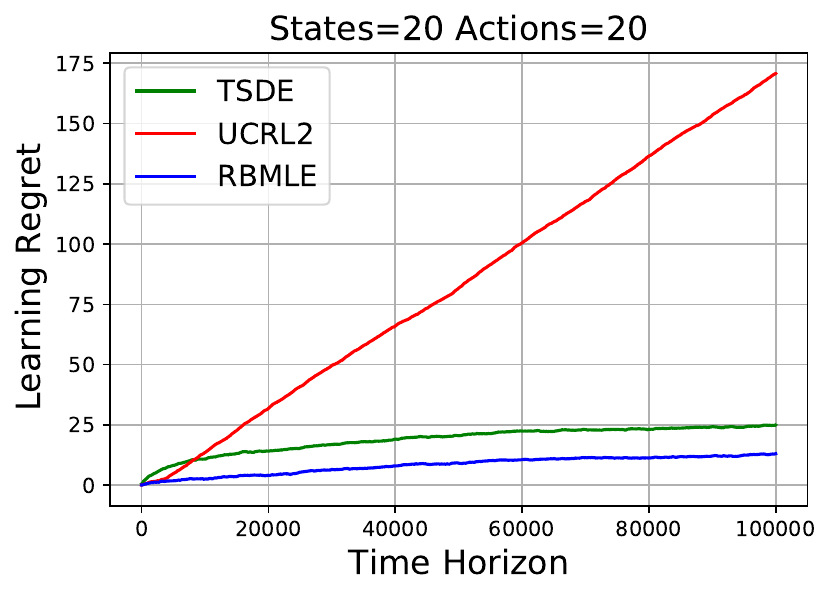}}
	\subfigure{
		\includegraphics[width=0.32\textwidth]{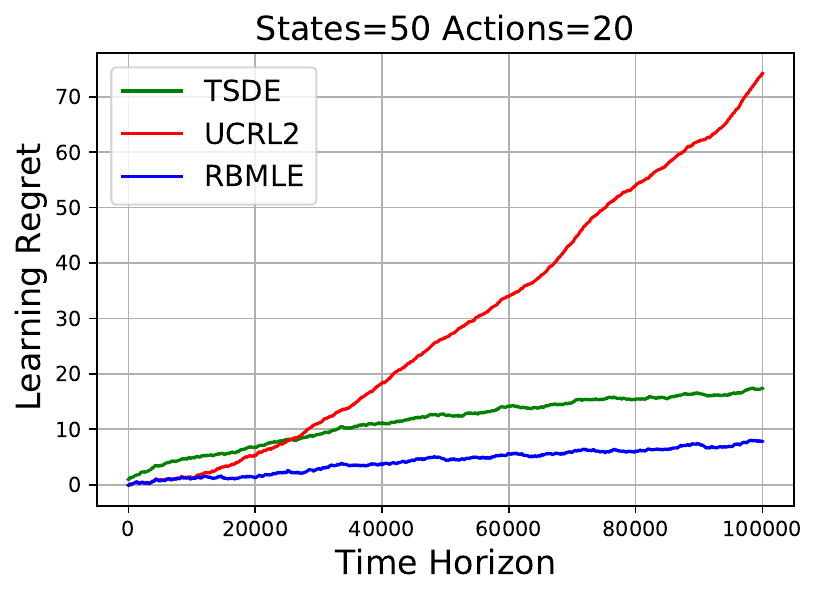}}
		\subfigure{
		\includegraphics[width=0.32\textwidth]{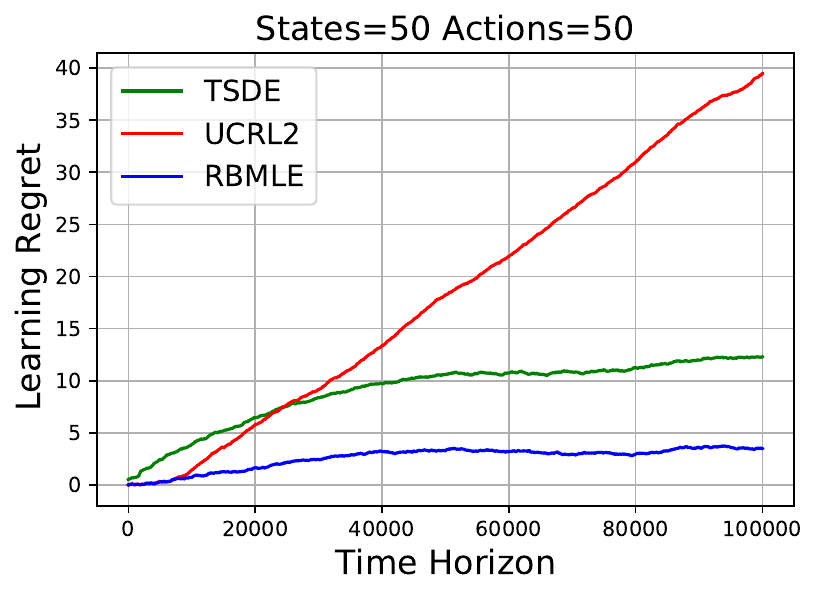}}
	{\caption{Empirical Comparision of RBMLE with UCRL2 and TSDE}\label{fig:1}}	
\end{figure}

		

		\section{Concluding Remarks}
		
		A fundamental challenge in online learning is what is prosaically called the closed-loop identifiability
		problem \citep{borkar_varaiya_79}. It is the same as the ``dual control" problem raised by
		\citep{feldbaum1960dual1,feldbaum1960dual2}, or the more contemporaneous 
		``exploration vs.~exploitation" problem: When a learning algorithm begins to converge, it ceases to learn or explore.
		It was noticed by \citep{kumar_becker_82} that in a certainty equivalence context this means that the
		learnt model will automatically have a one-sided bias of having a smaller optimal reward than the true model.
		RBMLE was proposed to overcome this fundamental problem by incorporating a counteracting bias in favor of parameters with larger optimal rewards. 
		It provides a general purpose reinforcement learning algorithm for dynamic stochastic systems.
		Most of the work on RBMLE has been focused on the problem of long-term average optimiality, the context in which it was originally proposed.
		However, current applications emphasize the much finer performance of regret, which captures the growth of the total reward
		as a function of the horizon $T$.
		Recent work examining RBMLE for stochastic bandits \citep{liu_20}, linear contextual bandits \citep{hung_20}
		has shown that not only does RBMLE have optimal order of regret but it also has excellent empirical performance
		competitive or better than state of the art algorithms, and it also achieves this with low computational complexity.
		For the LQG context, recent work motivated by RBMLE also establishes near optimal regret performance \citep{abbasi2011regret}.
		With the present paper establishing optimal order of regret for reinforcement learning problems modeled as Markov Decision Processes,
		the RBMLE complements the UCB approach and provides a second tool for reinforcement learning. 
		
\section*{Acknowledgments}
This research has been partially supported by NSF under CCF-1934904, Science \& Technology Center CCF-0939370, and CMMI-2038625, the USARO under W911NF-18-10331 and W911NF-2-120064, USARL W911NF-19-2-0243, and USONR N00014-18-1-2048. The views and
conclusions here do not represent the official policies, either expressed or implied, of NSF, USARO, USONR, USARL, or U.S. Government. The U.S. Government is authorized to reproduce and distribute reprints for Government purposes notwithstanding any copyright notation herein.					
		\bibliography{ref_file}
\appendix
\input{subsecs/appendices}

	\end{document}

%% file: subsecs/prelim.tex
\section{Preliminary Results}
	Define the following ``confidence interval'' $\cC(t)$ at time $t$ associated with the empirical estimate $\hat{p}(t)$,
	\begin{align}\label{def:d1}
	\cC(t): =
	\bigg\{ \te \in \Theta: |\te(x,y,u)-\hat{p}(x,y,u;t) |  \le d_1(x,u;t),\forall (x,y,u) \in X \times X \times U \bigg\},
	\end{align}
	where
	\begin{align}\label{eq:d1}
	d_1(x,u;t):= \sqrt{\frac{ \log\left( t^b |X|^{2} |U|\right) }{n(x,u;t)}  },
	\end{align}
and $b> 2$. Define also the set $\cG_1$,
\begin{align}\label{eq:goodset}
	\cG_1 := \left\{\omega:  p \in \cC(t), \; \forall\; t \in \mathbb{N} \right\}.
\end{align}
\begin{lemma}\label{lemma:confidence}
		The probability that $p$ lies in $\cC(t)$ is bounded as follows:
						\begin{align*}
			\bP(p \in \cC(t))> 1-\frac{2}{t^{2b-1}|X|^2|U|},\;\forall\;  t \in \mathbb{N}  .
		\end{align*}

	\end{lemma}
\begin{lemma} \label{lemma:pinkser} 
		\begin{align*}
		|\theta_{k,\pi}(x,y,u)-\hat{p}_k(x,y,u)|\le d_2(x,u;\tau_k), \;\forall\; (x,y,u) \in X \times X \times U,
		\end{align*}
	where
		\begin{align}\label{eq:d2}
	d_2(x,u;t):= \sqrt{\frac{\alpha(t)}{2n(x,u;t)}}, \forall (x,u)\in X\times U.
		\end{align}
		\end{lemma}
We now derive a lower bound on the index of any 
optimal stationary policy $\pi \ust \in \Pi \ust (p)$ that holds with
high probability. 
	\begin{lemma}\label{lemma:suff_prec_optimal}
On the set $\cG_1$, the index $I_k(\pi\ust)$ of any optimal policy $\pi \ust \in \pi \ust (p)$ is lower bounded as follows:
		\begin{align*}
		I_k(\pi\ust) \geq  \alpha(\tau_k)(1-\gamma)J\ust(p), ~\forall k =1,2,\ldots,K,
		\end{align*}
		where $\gamma := \frac{|X|^3|U|}{2a\cdot p_{\min} J \ust(p)}.$ 
	\end{lemma}
Next, we show that if the state-action pairs corresponding to a sub-optimal policy have been visited for a sufficiently large number of times, then its index $I_{k}(\pi)$ is lower than the index of any optimal policy. 
\begin{lemma}\label{lemma:subopt} 
	Let $\pi \notin \Pi \ust (p)$ be any sub-optimal stationary deterministic policy. Suppose that the number of visits to each $(x,\pi(x))$ until $\tau_k$ is lower bounded as follows,
	\begin{align}\label{eq:n_bound}
		n(x,\pi(x);\tau_k)> \frac{ \alpha(\tau_k)}{c^2} \; \forall x \in X,
	\end{align}where $c:=\frac{\beta\Delta_{\min}}{\kappa_p|X|^2\left(\frac{1}{\sqrt{2}}+\frac{1}{\sqrt{a}}\right)}$, $\beta \in \left(0,1-\gamma\frac{J\ust(p)}{\Delta_{\min}}\right)$ and $\gamma= \frac{|X|^3|U|}{2a p_{\min} J \ust(p)}$.
Then, the index of the sub-optimal policy $\pi$ at the beginning of the episode $k$ is strictly lower than the index of any optimal policy $\pi\ust$, i.e., $I_k(\pi) < I_k(\pi \ust)$.
\end{lemma}

%% file: subsecs/regret.tex
		\section{Regret Analysis}
		We begin by decomposing the cumulative regret of the learning rule $\phi$, $R(\phi,p,T)$ into the sum of episodic regrets as $R(\phi,p,T) := \sum_{k} \Big( J\ust(p) |\cE_k| - \sum_{t\in\cE_k} r(x(t),u(t)) \Big)$.
		Since the RBMLE algorithm implements stationary policy $\pi_k$ during $\cE_k$, we obtain the following bound on the expected $\bE R(\phi,p,T)$ (Lemma 12, \citeauthor{techreport}),
		\begin{align}\label{def:regret_dec}
		\bE  R(\phi,p,T) \le \sum_{k=1}^{K(T)} \bE \Big( J\ust(p) |\cE_k| - |\cE_k| J(\pi_k,p) \Big)  + T_pK(T),
		\end{align}
		where $K(T)$ is the number of episodes till $T$. The first summation can be regarded as the sum of the regrets arising from the policies chosen
		in the episodes $k=1,2, \ldots, K(T)$, assuming that each episode is started with a steady -state distribution for the state corresponding to the policy
		chosen in that episode. The last term $T_p K(T)$ can be regarded as the additional regret due to not starting in a steady state in each episode. We now present the main result of this paper which shows that the expected regret of the RBMLE algorithm is bounded by $c' \log T + c''$ for all $T$:
		\begin{theorem}\label{th:regret}
The regret of the RBMLE based-policy is upper-bounded as 
			\begin{align*}
			\bE  R(\phi,p,T) \le c_1\kappa_p^2|X|^5|U|\left(\frac{\sqrt{\frac{a}{2}}+1}{\beta \Delta_{\min} }\right)^2\log T+\left(\kappa_p|X||U|+1\right)\log_2 T+C \text{   for all } T,
			\end{align*}
			where 
			$\beta \in \left(0,1-\gamma\frac{J\ust(p)}{\Delta_{\min}}\right)$, $\gamma= \frac{|X|^3|U|}{2a p_{\min} J \ust(p)}$, $c_1 \in (0,\frac{11}{\kappa_p^2})$,
			 and
						\begin{align*}
			C=c_1\kappa_p^2|X|^5|U|\left(\frac{\sqrt{\frac{a}{2}}+1}{\beta \Delta_{\min} }\right)^2\log\left(|X|^{2} |U|\right)+\left(\kappa_p|X||U|+1\right)+|X||U|+\frac{8}{|X|^2|U|}.
		\end{align*}	
		\end{theorem}
	\begin{proof}
	The decomposition~\eqref{def:regret_dec} shows that: (a) Episodic regret is $0$ in those episodes in which $\pi_k$ is optimal, i.e., $\pi_k\in \Pi\ust(p)$. (b) if $\pi_k$ is not optimal then the episodic regret is bounded by the length of the episode $|\cE_k |$,
		since the magnitude of rewards is less than $1$. \\
Let ${\cG}_2$ denote the set with $\bP\left({\cG}_2 \right) \geq 1- \frac{|X||U|}{T}$ (Lemma 11, \citeauthor{techreport}) such that 
		\begin{align}\label{eq:g2}
			n(x,u;T) \ge  \frac{y_{x,u}}{2}  - \sqrt{y_{x,u}\log T} \text{ for all } (x,u) \text{ on } {\cG}_2,
		\end{align}
where $y_{x,u} := \sum_{k \in \cK_{x,u} } \Bigl\lfloor \frac{  |\cE_k | }{2 T_p} \Bigr\rfloor$  ,
and $\cK_{x,u}$ denotes the set of indices of those episodes up to time $T$ in which action $u$ is taken when state is equal to $x$.
Define the ``good set" $\cG := \cG_1 \cap {\cG}_2 $. We first consider the regret on $\cG$. \\
\textit{(i) Regret due to suboptimal episodes on $\cG$}: Define $n_c :=  \alpha(T)\left(\frac{\kappa_p|X|^2}{\beta\Delta_{\min}}\left(\frac{1}{\sqrt{2}}+\frac{1}{\sqrt{a}}\right)\right)^2$.

On the ``good set'' $\cG$, confidence intervals $\cC(t)$, defined in~\eqref{def:d1}, hold true for all episode starting times $\tau_k$, $k \in [1,K(T)]$, and also the conclusions of (\ref{eq:g2}) are true. Hence it follows from Lemma \ref{lemma:subopt} that if $n(x,u;\tau_k) >  n_c$ for all $(x,u)$ then the regret in $\cE_k$ is $0$. Otherwise, there exists at least one state, action pair $(x,u)$ with $n(x,u;\tau_k) \le n_c$. We now upper bound the number of time-steps in such ``sub-optimal'' episodes $\cK_{x,u}$ in which control $u$ is applied in state $x$. 	Since $n(x,u;T)\le n_c$, we have $n_c \ge  \frac{y_{x,u}}{2}  - \sqrt{y_{x,u}\log T}$. Note that $n_c \ge\kappa_p^2\log T$.  Then, there exists $c_1<\frac{11}{\kappa^2_p} $ such that $y_{x,u} \leq c_1n_c$ (Lemma 13, \citeauthor{techreport}).
		So
		$
			\sum_{k \in \cK_{x,u} }  |\cE_k |  \le 2 T_pc_1n_c+2|\cK_{x,u}|T_p
		$.
		Note that $\cK_{x,u}<K(T)$, where $K(T)$ is the total number of episodes till $T$. We let $\mathcal{R}_1$ be the total regret until $T$ due to suboptimal episodes on the good set $\mathcal{G}$. Then 
		\begin{align}\label{ineq:r1}
			\mathcal{R}_1  \le \sum_{(x,u)}\sum_{k \in \cK_{x,u} } |\cE_k |\le 2 |X||U|T_p(c_1n_c+K(T)),
		\end{align}
		where $\mathcal{R}_1$ is the expected regret on $\cG$.
	\\\\
	\textit{(ii) Regret on $\cG_1^c$}: For any episode $k \in [0,K(T)]$, the probability of failure of the confidence interval at the beginning of the episode $C(\tau_k)$ is upper bounded by $\frac{2}{|X|^2|U|\tau_k^{2b-1}}$ as shown in Lemma~\ref{lemma:confidence}.  
			The expected regret in each such episode is bounded by the length of the episode $\cE_k$. We let $\mathcal{R}_3$ be the total expected regret till $T$ due to the failure of confidence intervals. It can be upper-bounded as:			
\begin{align*}
				\mathcal{R}_3& \leq \sum_{k=1}^{K(T)} \frac{2(\tau_{k+1}-\tau_{k})}{|X|^2|U|\tau_k^{2b-1}}\le \sum_{k=1}^{\infty} \frac{2}{|X|^2|U|\tau_k^{2b-2}}\left(\frac{\tau_{k+1}}{\tau_k}\right)		\le \sum_{k=1}^\infty \frac{4}{|X|^2|U|\tau_k^{2b-2}} \le \frac{8}{|X|^2|U|}.
\end{align*}
\textit{(iii) Regret on ${\cG}_2^c$}: The probability of the set where the conclusions of (\ref{eq:g2}) do not hold true for a state-action pair $(x,u)$ is upper bounded by $\frac{|X||U|}{T}$. Since the sample-path regret can be trivially upper bounded by $T$, it follows that $\mathcal{R}_2\le |X||U|$. \\\\
\textit{(iv) Additional regret due to not starting in a steady state in each episode}: The RBMLE algorithm implements the policy $\pi_k$ at the beginning of episodes $k$. The total expected reward in the episode depends on the state at the beginning of the episode $x(\tau_
k)$ (Lemma 12, \citeauthor{techreport}). The policy incurs an additional loss if it starts in an unfavorable state at $\tau_k$ which is upper bounded by $T_p$ in each episode. Let $\mathcal{R}_4$ be the total expected regret due to not starting in a steady state in each episode. Then $\mathcal{R}_4 \le K(T)T_p$ where, $K(T)=\lceil\log_2 T \rceil$ is the total number of episodes till $T$.\\\\
The proof is completed by adding the bounds on $\mathcal{R}_1,\mathcal{R}_2,\mathcal{R}_3$ and $\mathcal{R}_4$.
\end{proof}

%% file: subsecs/appendices.tex
\section{Proof of Lemma \ref{lemma:confidence}}\label{proof:confidence}
Consider the scenario where the number of visits to $(x,u)$ is fixed at $n_{x,u}$, and let $\hat{p}(x,y,u)$ be the resulting estimates. Consider the event $\left\{|p(x,y,u)-\hat{p}(x,y,u) | > r\right\}$, where $x,y \in X, u \in U$ and $r>0$. It follows from the Azuma-Hoeffding's inequality \citep{mitzenmacher2017probability} that the probability of this event is upper bounded by $2\exp(-2n_{x,u}r^2)$. Therefore, 
\begin{align*}
	\bP\left(|p(x,y,u)-\hat{p}(x,y,u) | > \sqrt{\frac{ \log\left(t^b |X|^{2} |U|\right) }{n_{x,u}}   }     \right) \leq 2\left(\frac{1}{t^b|X|^{2}|U|}\right)^2.
\end{align*} 
Utilizing union bound on the number of plays of action $u$ in state $x$ until time $t$ and considering all possible state-action-state pairs, we get
\begin{align*}
	\bP(p \notin \cC(t))\le \frac{2}{|X|^2|U|t^{2b-1}}\;\forall\; t \in [1,T] .
\end{align*}
\section{Proof of Lemma \ref{lemma:pinkser}} \label{proof:pinsker}
	The index of the policy $\pi$ (\ref{eq:index}) can be written as:  
\begin{align*}
	I_k(\pi)&= \alpha(\tau_k) J(\theta_{k,\pi},\pi) -\sum_{(x,u)} n_k(x,u)KL\left( \hat{p}_{k}(x,u),\theta_{k,\pi}(x,u)  \right)\\&\ge \alpha(\tau_k) J(\hat{p}_k,\pi).
\end{align*}
Since the average reward $J(\theta,\pi) \in [0,1]$
for all $\theta \in \Theta$ and $\pi \in \Pi_{sd}$, we get
\begin{align*}
	n_k(x,u)KL\left( \hat{p}_{k}(x,u),\theta_{k,\pi}(x,u)  \right) \le  \alpha(\tau_k) ( J(\theta_{k,\pi},\pi_k)-J(\hat{p}_k,\pi_k)) \le   \alpha(\tau_k) \;\forall \; (x,u) \in X \times U. 
\end{align*}
By using Pinsker's inequality~\citep{cover1999elements}, we can bound KL-divergence as follows:
\begin{align*}
	|\theta_k(x,y,u)-\hat{p}_k(x,y,u)|^2 \le   \frac{1}{2}KL\left( \hat{p}_{k}(x,u),\te_k(x,u)  \right) \; \forall x,y \in X \text{ and } u \in U.
\end{align*}
The proof is completed by substituting this bound into the above inequality.
\section{Proof of Lemma \ref{lemma:suff_prec_optimal}}\label{proof:suff_prec_optimal}
	The RBMLE index of an optimal policy $\pi \ust$ (\ref{eq:index}) satisfies ,
\begin{align*}
	I_k(\pi \ust) & = \bigg\{  \alpha(\tau_k) J (\theta_{k,\pi \ust},\pi \ust) -\sum_{(x,u)} n_k(x,u)KL\left( \hat{p}_{k}(x,u),\te_{k,\pi\ust}(x,u)  \right) \bigg\}\\ 
	&\ge \bigg\{  \alpha(\tau_k)J \ust(p) -\sum_{(x,u)} n_k(x,u)KL\left( \hat{p}_{k}(x,u),p(x,u)  \right) \bigg\}\\ 
	&\ge 
	\bigg\{  \alpha(\tau_k)J \ust(p) -\sum_{(x,u)} n_k(x,u)\frac{\big(\sum\limits_{y\in X} |p(x,y,u)-\hat{p}_{k}(x,y,u)|\big)^2}{2p_{\min}} \bigg\},
\end{align*}
where the first inequality follows since $\te_{k,\pi\ust}$ maximizes the objective in~\eqref{eq:index}, 
while the second inequality follows from the inverse Pinkser's inequality \citep{cover1999elements} and Assumption~\ref{assum:1}.
Since on $\cG_1$, we have that $|p(x,y,u)-\hat{p}_{k}(x,y,u)|<d_1(x,u;t)$  for all $(x,y,u) \in X \times X \times U$, it follows that
\begin{align*}
	I_k(\pi \ust)\ge 	\bigg\{  \alpha(\tau_k)J \ust(p) - \frac{|X|^2|U|}{2p_{\min}}\log\left(t^b |X|^{2} |U|\right) \bigg\}
	= \alpha(\tau_k)J \ust(p)\left(1-\frac{|X|^{2}|U|}{2a p_{\min} J \ust(p)}\right).
\end{align*}

\section{Proof of Lemma \ref{lemma:subopt}}\label{proof:subopt}
\textit{(i)}
As is shown in Lemma \ref{lemma:confidence}, Lemma \ref{lemma:pinkser}, the distance between $p(x,y,u)$ and $\hat{p}(x,y,u)$ can be bounded by $d_1(x,u;\tau_k)$ while the distance between $\theta_k(x,y,u)$ and $\hat{p}(x,y,u)$ can be bounded by $d_2(x,u;\tau_k)$. The proof then follows from the triangle inequality.\\\\
\textit{(ii)}
The index of the stationary policy $\pi$ can be written as follows (\ref{eq:index}),
\begin{align}
	I_k(\pi)&=  \alpha(\tau_k) J(\theta_{k,\pi},\pi) -\sum_{(x,u)} n_k(x,u)KL\left( \hat{p}_{k}(x,u),\theta_{k,\pi}(x,u)  \right)\notag\\&\le  \alpha(\tau_k) J(\theta_{k,\pi},\pi).\label{eq:up_b1}
\end{align}
If (\ref{eq:n_bound}) holds then the distance between $\theta_{k,\pi}$ and true transition probability $p$ can be bounded as follows (Lemma \ref{lemma:subopt}, (i)):
\begin{align*}
	|\theta_{k,\pi}(x,y,\pi(x))-p(x,y,\pi(x))| < c \left(\frac{1}{\sqrt{2}}+\frac{1}{\sqrt{a}}\right) \;\forall\;x,y \in X.
\end{align*}
Then the average reward $J(\theta_{k,\pi},\pi)$ can be bounded using Lemma \ref{lemma:reward_bound} as follows:
\begin{align}\label{eq:reward}
	J(\theta_{k,\pi},\pi)< J(p,\pi)+c\kappa_p|X|^2\left(\frac{1}{\sqrt{2}}+\frac{1}{\sqrt{a}}\right)=J(p,\pi)+\beta\Delta_{\min}.
\end{align} 
The result follows from (\ref{eq:up_b1}) and (\ref{eq:reward}).\\\\
\textit{(iii)} It follows from (i) and (iii) that if (\ref{eq:n_bound}) holds then it is sufficient to show that 
\begin{align*}
	J\ust(p)  \left(1-\gamma\right) \ge  J(p,\pi) + \beta \Delta_{\min},
\end{align*}  
which holds true since $J\ust(p)-J(p,\pi) \geq \Delta_{\min}$ and $\beta <1-\gamma\frac{J\ust(p)}{\Delta_{\min}} $.
\section{Auxiliary Results}
The following results are from \citep{cho_00} and \citep{auer_07} respectively.
\begin{lemma}\citep{cho_00}\label{lemma:reward_bound} Consider a stationary policy $\pi$ and $\te$ be an MDP parameter that satisfies
	\begin{align}\label{def:suff_visits}
		|\te(x,y,\pi(x))-p(x,y,\pi(x))|< \frac{\epsilon}{\kappa_p |X|^{2}}, \; \forall \; x,y \in X,
	\end{align}
	where $\epsilon>0$ and $\kappa_p$ is the conductivity. We then have that 
	\begin{align*}
		| J(\te,\pi) - J(p,\pi) | < \epsilon.
	\end{align*}
\end{lemma}
\begin{lemma}\citep{auer_07}\label{lemma:azum_visits}
	Let $\cK_{x,u}$ denote the indices of those episodes up to time $T$ in which action $u$ is taken when state is equal to $x$. Then
	\begin{align}\label{eq:azum_visits}
		\bP\left(  n(x,u;T) \ge  \frac{y_{x,u}}{2}  - \sqrt{y_{x,u}\log T} ~ \forall~{x,u} \right) \ge 1 - \frac{|X||U|}{T},
	\end{align}
	for all state-action pairs $(x,u)$, where
	\begin{align*}
		y_{x,u} := \sum_{k \in \cK_{x,u} } \Bigl\lfloor \frac{  |\cE_k| }{2 T_p} \Bigr\rfloor .
	\end{align*}
\end{lemma}
\begin{lemma}(Lemma 2, \cite{auer_07})\label{lemma:mix}
	Let $\pi$ be a stationary policy. Consider a controlled Markov process that starts in state $x$ and evolves under $\pi$. We then have that 
	\begin{align*}
		\bE_{x}\left( \sum_{t=1}^{T} r(x(t),u(t)) \right)  \ge TJ(\pi,p) - T_p.
	\end{align*}
\end{lemma}
\begin{lemma}\label{lemma:bound_on_c}
	Consider the following function $f(x)$ such that $a_0>a_1>0$,
	\begin{align}
		f(x)=x-2\sqrt{a_1x}-2a_0.
	\end{align}
Then there exist $x_0<11a_0$ such that $f(x)>0$ for all $x>x_0$.
\end{lemma}
\begin{proof}Note that $f(a_1)<0$ and 
	\begin{align*}
\frac{\partial f}{\partial x}=1-\sqrt{\frac{a_1}{x}}>0 ~\forall~ x>a_1.
\end{align*}
The result follows since $f(11a_0)=9a_0-2\sqrt{11a_0a_1}>(9-2\sqrt{11})a_0>0.$
\end{proof}

%% file: main.bbl
\begin{thebibliography}{41}
\providecommand{\natexlab}[1]{#1}
\providecommand{\url}[1]{\texttt{#1}}
\expandafter\ifx\csname urlstyle\endcsname\relax
  \providecommand{\doi}[1]{doi: #1}\else
  \providecommand{\doi}{doi: \begingroup \urlstyle{rm}\Url}\fi

\bibitem[Abbasi-Yadkori and Szepesv{\'a}ri(2011)]{abbasi2011regret}
Yasin Abbasi-Yadkori and Csaba Szepesv{\'a}ri.
\newblock Regret bounds for the adaptive control of linear quadratic systems.
\newblock In \emph{Proceedings of the 24th Annual Conference on Learning
  Theory}, pages 1--26, 2011.

\bibitem[Abbasi-Yadkori and Szepesv{\'a}ri(2015)]{abbasi_15}
Yasin Abbasi-Yadkori and Csaba Szepesv{\'a}ri.
\newblock Bayesian optimal control of smoothly parameterized systems.
\newblock In \emph{UAI}, pages 1--11. Citeseer, 2015.

\bibitem[Agrawal and Goyal(2013)]{agrawa_13}
Shipra Agrawal and Navin Goyal.
\newblock Thompson sampling for contextual bandits with linear payoffs.
\newblock In \emph{International Conference on Machine Learning}, pages
  127--135, 2013.

\bibitem[Auer and Ortner(2007)]{auer_07}
Peter Auer and Ronald Ortner.
\newblock Logarithmic online regret bounds for undiscounted reinforcement
  learning.
\newblock In \emph{Advances in Neural Information Processing Systems}, pages
  49--56, 2007.

\bibitem[Auer et~al.(2002)Auer, Cesa-Bianchi, and Fischer]{auer_02}
Peter Auer, Nicolo Cesa-Bianchi, and Paul Fischer.
\newblock Finite-time analysis of the multiarmed bandit problem.
\newblock \emph{Machine learning}, 47\penalty0 (2-3):\penalty0 235--256, 2002.

\bibitem[Auer et~al.(2009)Auer, Jaksch, and Ortner]{auer_09}
Peter Auer, Thomas Jaksch, and Ronald Ortner.
\newblock Near-optimal regret bounds for reinforcement learning.
\newblock In \emph{Advances in neural information processing systems}, pages
  89--96, 2009.

\bibitem[Bartlett and Tewari(2012)]{bartlett_12}
Peter~L. Bartlett and Ambuj Tewari.
\newblock {REGAL:} {A} regularization based algorithm for reinforcement
  learning in weakly communicating mdps.
\newblock \emph{CoRR}, abs/1205.2661, 2012.
\newblock URL \url{http://arxiv.org/abs/1205.2661}.

\bibitem[Becker and Kumar(1981)]{becker_81}
A.~Becker and P.~R. Kumar.
\newblock Optimal strategies for the n-armed bandit problem.
\newblock \emph{Univ. Maryland. Baltimore County, Math. Res. Rep}, pages 81--1,
  1981.

\bibitem[Berry and Fristedt(1985)]{berry1985bandit}
Donald~A. Berry and Bert Fristedt.
\newblock Bandit problems: sequential allocation of experiments (monographs on
  statistics and applied probability).
\newblock \emph{London: Chapman and Hall}, 5\penalty0 (71-87):\penalty0 7--7,
  1985.

\bibitem[Bittanti et~al.(2006)Bittanti, Campi, et~al.]{bittanti2006adaptive}
Sergio Bittanti, Marco~C. Campi, et~al.
\newblock Adaptive control of linear time invariant systems: the ``bet on the
  best" principle.
\newblock \emph{Communications in Information \& Systems}, 6\penalty0
  (4):\penalty0 299--320, 2006.

\bibitem[{Borkar} and {Varaiya}(1979)]{borkar_varaiya_79}
V.~{Borkar} and P.~{Varaiya}.
\newblock {Adaptive control of {M}arkov chains, I: Finite parameter set}.
\newblock \emph{IEEE Transactions on Automatic Control}, 24\penalty0
  (6):\penalty0 953--957, 1979.

\bibitem[Borkar(1990)]{borkar1990kumar}
V.~S. Borkar.
\newblock The {K}umar-{B}ecker-{L}in scheme revisited.
\newblock \emph{Journal of Optimization Theory and Applications}, 66\penalty0
  (2):\penalty0 289--309, 1990.

\bibitem[Borkar(1991)]{borkar1991self}
V.~S. Borkar.
\newblock Self-tuning control of diffusions without the identifiability
  condition.
\newblock \emph{Journal of optimization theory and applications}, 68\penalty0
  (1):\penalty0 117--138, 1991.

\bibitem[Brafman and Tennenholtz(2002)]{brafman_02}
Ronen~I Brafman and Moshe Tennenholtz.
\newblock R-max-a general polynomial time algorithm for near-optimal
  reinforcement learning.
\newblock \emph{Journal of Machine Learning Research}, 3\penalty0
  (Oct):\penalty0 213--231, 2002.

\bibitem[Campi and Kumar(1998)]{campi_98}
Marco~C. Campi and P.~R. Kumar.
\newblock Adaptive linear quadratic {G}aussian control: the cost-biased
  approach revisited.
\newblock \emph{SIAM Journal on Control and Optimization}, 36\penalty0
  (6):\penalty0 1890--1907, 1998.

\bibitem[Cho and Meyer(2000)]{cho_00}
Grace~E. Cho and Carl~D. Meyer.
\newblock Markov chain sensitivity measured by mean first passage times.
\newblock \emph{Linear Algebra and its Applications}, 316\penalty0
  (1-3):\penalty0 21--28, 2000.

\bibitem[Cover(1999)]{cover1999elements}
Thomas~M. Cover.
\newblock \emph{Elements of information theory}.
\newblock John Wiley \& Sons, 1999.

\bibitem[Duncan et~al.(1994)Duncan, Pasik-Duncan, and
  Stettner]{duncan1994almost}
T.~E. Duncan, B.~Pasik-Duncan, and L.~Stettner.
\newblock Almost self-optimizing strategies for the adaptive control of
  diffusion processes.
\newblock \emph{Journal of optimization theory and applications}, 81\penalty0
  (3):\penalty0 479--507, 1994.

\bibitem[Feldbaum(1960{\natexlab{a}})]{feldbaum1960dual1}
A.~A. Feldbaum.
\newblock Dual control theory. i.
\newblock \emph{Avtomatika i Telemekhanika}, 21\penalty0 (9):\penalty0
  1240--1249, 1960{\natexlab{a}}.

\bibitem[Feldbaum(1960{\natexlab{b}})]{feldbaum1960dual2}
A.~A. Feldbaum.
\newblock Dual control theory. ii.
\newblock \emph{Avtomatika i Telemekhanika}, 21\penalty0 (11):\penalty0
  1453--1464, 1960{\natexlab{b}}.

\bibitem[Gopalan and Mannor(2015)]{gopalan_15}
Aditya Gopalan and Shie Mannor.
\newblock Thompson sampling for learning parameterized {M}arkov decision
  processes.
\newblock In \emph{Conference on Learning Theory}, pages 861--898, 2015.

\bibitem[Hung et~al.(2020)]{hung_20}
Y-H. Hung et~al.
\newblock Reward-biased maximum likelihood estimation for linear stochastic
  bandits.
\newblock \emph{arXiv preprint arXiv:2010.04091}, 2020.

\bibitem[Kumar(1982)]{kumar_82}
P.~R. Kumar.
\newblock Adaptive control with a compact parameter set.
\newblock \emph{SIAM Journal on Control and Optimization}, 20\penalty0
  (1):\penalty0 9--13, 1982.

\bibitem[Kumar(1983{\natexlab{a}})]{kumar1983optimal}
P.~R. Kumar.
\newblock Optimal adaptive control of linear-quadratic-{G}aussian systems.
\newblock \emph{SIAM Journal on Control and Optimization}, 21\penalty0
  (2):\penalty0 163--178, 1983{\natexlab{a}}.

\bibitem[Kumar(1983{\natexlab{b}})]{kumar1983simultaneous}
P.~R. Kumar.
\newblock Simultaneous identification and adaptive control of unknown systems
  over finite parameter sets.
\newblock \emph{IEEE Transactions on Automatic Control}, 28\penalty0
  (1):\penalty0 68--76, 1983{\natexlab{b}}.

\bibitem[Kumar and Becker(1982)]{kumar_becker_82}
P.~R. Kumar and A.~Becker.
\newblock A new family of optimal adaptive controllers for {M}arkov chains.
\newblock \emph{IEEE Transactions on Automatic Control}, 27\penalty0
  (1):\penalty0 137--146, 1982.

\bibitem[Kumar and Lin(1982)]{kumar_lin_82}
P.~R. Kumar and W.~Lin.
\newblock Optimal adaptive controllers for unknown {M}arkov chains.
\newblock \emph{IEEE Transactions on Automatic Control}, 27\penalty0
  (4):\penalty0 765--774, 1982.

\bibitem[Lai and Robbins(1985)]{lai_85}
Tze~Leung Lai and Herbert Robbins.
\newblock Asymptotically efficient adaptive allocation rules.
\newblock \emph{Advances in applied mathematics}, 6\penalty0 (1):\penalty0
  4--22, 1985.

\bibitem[Liu et~al.(2020)Liu, Hsieh, Hung, Bhattacharya, and Kumar]{liu_20}
Xi~Liu, Ping-Chun Hsieh, Yu~Heng Hung, Anirban Bhattacharya, and P~Kumar.
\newblock Exploration through reward biasing: Reward-biased maximum likelihood
  estimation for stochastic multi-armed bandits.
\newblock In \emph{International Conference on Machine Learning}, pages
  6248--6258. PMLR, 2020.

\bibitem[Mandl(1974)]{mandl_74}
P.~Mandl.
\newblock Estimation and control in {M}arkov chains.
\newblock \emph{Advances in Applied Probability}, pages 40--60, 1974.

\bibitem[Mete et~al.(2020)Mete, Singh, and Kumar]{techreport}
Akshay Mete, Rahul Singh, and P.~R. Kumar.
\newblock {Reward Biased Maximum Likelihood Estimation for Reinforcement
  Learning}, 2020.
\newblock URL \url{https://arxiv.org/abs/2011.07738}.

\bibitem[Mitzenmacher and Upfal(2017)]{mitzenmacher2017probability}
Michael Mitzenmacher and Eli Upfal.
\newblock \emph{Probability and computing: Randomization and probabilistic
  techniques in algorithms and data analysis}.
\newblock Cambridge university press, 2017.

\bibitem[Osband et~al.(2013)Osband, Russo, and Van~Roy]{osband_13}
Ian Osband, Daniel Russo, and Benjamin Van~Roy.
\newblock {(More) efficient} reinforcement learning via posterior sampling.
\newblock In \emph{Advances in Neural Information Processing Systems}, pages
  3003--3011, 2013.

\bibitem[Ouyang et~al.(2017)Ouyang, Gagrani, Nayyar, and Jain]{ouyang_17}
Yi~Ouyang, Mukul Gagrani, Ashutosh Nayyar, and Rahul Jain.
\newblock Learning unknown {M}arkov decision processes: A thompson sampling
  approach.
\newblock \emph{arXiv preprint arXiv:1709.04570}, 2017.

\bibitem[Prandini and Campi(2000)]{prandini_00}
Maria Prandini and Marco~C. Campi.
\newblock Adaptive lqg control of input-output systems---a cost-biased
  approach.
\newblock \emph{SIAM Journal on Control and Optimization}, 39\penalty0
  (5):\penalty0 1499--1519, 2000.

\bibitem[Puterman(2014)]{puterman2014markov}
Martin~L. Puterman.
\newblock \emph{Markov decision processes: discrete stochastic dynamic
  programming}.
\newblock John Wiley \& Sons, 2014.

\bibitem[Singh et~al.(2020)Singh, Gupta, and Shroff]{singh2020learning}
Rahul Singh, Abhishek Gupta, and Ness~B. Shroff.
\newblock Learning in {M}arkov decision processes under constraints.
\newblock \emph{arXiv preprint arXiv:2002.12435}, 2020.

\bibitem[Srinivas et~al.(2010)Srinivas, Krause, Kakade, and
  Seeger]{srinivas_09}
Niranjan Srinivas, Andreas Krause, Sham~M. Kakade, and Matthias Seeger.
\newblock Gaussian process optimization in the bandit setting: No regret and
  experimental design.
\newblock \emph{27-th International Conference on Machine Learning}, 2010.

\bibitem[Stettner(1993)]{stettner1993nearly}
Lukasz Stettner.
\newblock On nearly self-optimizing strategies for a discrete-time uniformly
  ergodic adaptive model.
\newblock \emph{Applied Mathematics and Optimization}, 27\penalty0
  (2):\penalty0 161--177, 1993.

\bibitem[Strens(2000)]{strens_00}
Malcolm Strens.
\newblock A {B}ayesian framework for reinforcement learning.
\newblock In \emph{ICML}, volume 2000, pages 943--950, 2000.

\bibitem[Sutton et~al.(1998)Sutton, Barto, et~al.]{sutton1998introduction}
Richard~S. Sutton, Andrew~G. Barto, et~al.
\newblock \emph{Introduction to reinforcement learning}, volume 135.
\newblock MIT press Cambridge, 1998.

\end{thebibliography}


\begin{thebibliography}{37}
\providecommand{\natexlab}[1]{#1}
\providecommand{\url}[1]{\texttt{#1}}
\expandafter\ifx\csname urlstyle\endcsname\relax
  \providecommand{\doi}[1]{doi: #1}\else
  \providecommand{\doi}{doi: \begingroup \urlstyle{rm}\Url}\fi

\bibitem[Abbasi-Yadkori and Szepesv{\'a}ri(2011)]{abbasi2011regret}
Yasin Abbasi-Yadkori and Csaba Szepesv{\'a}ri.
\newblock Regret bounds for the adaptive control of linear quadratic systems.
\newblock In \emph{Proceedings of the 24th Annual Conference on Learning
  Theory}, pages 1--26, 2011.

\bibitem[Abeille and Lazaric(2018)]{abeille_18}
Marc Abeille and Alessandro Lazaric.
\newblock {Improved Regret Bounds for Thompson Sampling in Linear Quadratic
  Control Problems}.
\newblock \emph{{Proceedings of Machine Learning Research}}, 80, 2018.
\newblock URL \url{https://hal.archives-ouvertes.fr/hal-02767726}.

\bibitem[Agrawal and Goyal(2013)]{agrawa_13}
Shipra Agrawal and Navin Goyal.
\newblock Thompson sampling for contextual bandits with linear payoffs.
\newblock In \emph{International Conference on Machine Learning}, pages
  127--135, 2013.

\bibitem[Auer and Ortner(2007)]{auer_07}
Peter Auer and Ronald Ortner.
\newblock Logarithmic online regret bounds for undiscounted reinforcement
  learning.
\newblock In \emph{Advances in Neural Information Processing Systems}, pages
  49--56, 2007.

\bibitem[Auer et~al.(2002)Auer, Cesa-Bianchi, and Fischer]{auer_02}
Peter Auer, Nicolo Cesa-Bianchi, and Paul Fischer.
\newblock Finite-time analysis of the multiarmed bandit problem.
\newblock \emph{Machine learning}, 47\penalty0 (2-3):\penalty0 235--256, 2002.

\bibitem[Auer et~al.(2009)Auer, Jaksch, and Ortner]{auer_09}
Peter Auer, Thomas Jaksch, and Ronald Ortner.
\newblock Near-optimal regret bounds for reinforcement learning.
\newblock In \emph{Advances in neural information processing systems}, pages
  89--96, 2009.

\bibitem[Bartlett and Tewari(2012)]{bartlett_12}
Peter~L. Bartlett and Ambuj Tewari.
\newblock {REGAL:} {A} regularization based algorithm for reinforcement
  learning in weakly communicating mdps.
\newblock \emph{CoRR}, abs/1205.2661, 2012.
\newblock URL \url{http://arxiv.org/abs/1205.2661}.

\bibitem[Becker and Kumar(1981)]{becker_81}
A.~Becker and P.R. Kumar.
\newblock Optimal strategies for the n-armed bandit problem.
\newblock \emph{Univ. Maryland. Baltimore County, Math. Res. Rep}, pages 81--1,
  1981.

\bibitem[Berry and Fristedt(1985)]{berry1985bandit}
Donald~A Berry and Bert Fristedt.
\newblock Bandit problems: sequential allocation of experiments (monographs on
  statistics and applied probability).
\newblock \emph{London: Chapman and Hall}, 5\penalty0 (71-87):\penalty0 7--7,
  1985.

\bibitem[{Borkar} and {Varaiya}(1979)]{borkar_varaiya_79}
V.~{Borkar} and P.~{Varaiya}.
\newblock Adaptive control of markov chains, i: Finite parameter set.
\newblock \emph{IEEE Transactions on Automatic Control}, 24\penalty0
  (6):\penalty0 953--957, 1979.

\bibitem[Brafman and Tennenholtz(2002)]{brafman_02}
Ronen~I Brafman and Moshe Tennenholtz.
\newblock R-max-a general polynomial time algorithm for near-optimal
  reinforcement learning.
\newblock \emph{Journal of Machine Learning Research}, 3\penalty0
  (Oct):\penalty0 213--231, 2002.

\bibitem[Campi and Kumar(1998)]{campi_98}
Marco~C Campi and PR~Kumar.
\newblock Adaptive linear quadratic gaussian control: the cost-biased approach
  revisited.
\newblock \emph{SIAM Journal on Control and Optimization}, 36\penalty0
  (6):\penalty0 1890--1907, 1998.

\bibitem[Cho and Meyer(2000)]{cho_00}
Grace~E Cho and Carl~D Meyer.
\newblock Markov chain sensitivity measured by mean first passage times.
\newblock \emph{Linear Algebra and its Applications}, 316\penalty0
  (1-3):\penalty0 21--28, 2000.

\bibitem[Cohen et~al.(2019)Cohen, Koren, and Mansour]{Cohen_19}
Alon Cohen, Tomer Koren, and Yishay Mansour.
\newblock Learning linear-quadratic regulators efficiently with only
  {\textdollar}{\textbackslash}sqrt\{T\}{\textdollar} regret.
\newblock \emph{CoRR}, abs/1902.06223, 2019.
\newblock URL \url{http://arxiv.org/abs/1902.06223}.

\bibitem[Cover(1999)]{cover1999elements}
Thomas~M Cover.
\newblock \emph{Elements of information theory}.
\newblock John Wiley \& Sons, 1999.

\bibitem[Dean et~al.(2018)Dean, Mania, Matni, Recht, and Tu]{dean_18}
Sarah Dean, Horia Mania, Nikolai Matni, Benjamin Recht, and Stephen Tu.
\newblock Regret bounds for robust adaptive control of the linear quadratic
  regulator.
\newblock In S.~Bengio, H.~Wallach, H.~Larochelle, K.~Grauman, N.~Cesa-Bianchi,
  and R.~Garnett, editors, \emph{Advances in Neural Information Processing
  Systems 31}, pages 4188--4197. Curran Associates, Inc., 2018.

\bibitem[Filippi et~al.(2010)Filippi, Capp{\'e}, and Garivier]{filippi_10}
Sarah Filippi, Olivier Capp{\'e}, and Aur{\'e}lien Garivier.
\newblock Optimism in reinforcement learning and kullback-leibler divergence.
\newblock In \emph{2010 48th Annual Allerton Conference on Communication,
  Control, and Computing (Allerton)}, pages 115--122. IEEE, 2010.

\bibitem[Gopalan and Mannor(2015)]{gopalan_15}
Aditya Gopalan and Shie Mannor.
\newblock Thompson sampling for learning parameterized markov decision
  processes.
\newblock In \emph{Conference on Learning Theory}, pages 861--898, 2015.

\bibitem[Hung et~al.(2020)Hung, Hsieh, Liu, and Kumar]{hung_20}
Yu-Heng Hung, Ping-Chun Hsieh, Xi~Liu, and P.~R. Kumar.
\newblock Reward-biased maximum likelihood estimation for linear stochastic
  bandits, 2020.

\bibitem[{Kumar} and {Becker}(1982)]{kumar_becker_82}
P.~{Kumar} and A.~{Becker}.
\newblock A new family of optimal adaptive controllers for markov chains.
\newblock \emph{IEEE Transactions on Automatic Control}, 27\penalty0
  (1):\penalty0 137--146, 1982.

\bibitem[{Kumar} and {Woei Lin}(1982)]{kumar_lin_82}
P.~{Kumar} and {Woei Lin}.
\newblock Optimal adaptive controllers for unknown markov chains.
\newblock \emph{IEEE Transactions on Automatic Control}, 27\penalty0
  (4):\penalty0 765--774, 1982.

\bibitem[Kumar(1982)]{kumar_82}
P.R. Kumar.
\newblock Adaptive control with a compact parameter set.
\newblock \emph{SIAM Journal on Control and Optimization}, 20\penalty0
  (1):\penalty0 9--13, 1982.

\bibitem[Kumar and Varaiya(2015)]{kumar2015stochastic}
P.R. Kumar and P.~Varaiya.
\newblock \emph{Stochastic systems: Estimation, identification, and adaptive
  control}.
\newblock SIAM, 2015.

\bibitem[Lai and Robbins(1985)]{lai_85}
Tze~Leung Lai and Herbert Robbins.
\newblock Asymptotically efficient adaptive allocation rules.
\newblock \emph{Advances in applied mathematics}, 6\penalty0 (1):\penalty0
  4--22, 1985.

\bibitem[Liu et~al.(2020)Liu, Hsieh, Hung, Bhattacharya, and Kumar]{liu_20}
Xi~Liu, Ping-Chun Hsieh, Yu-Heng Hung, Anirban Bhattacharya, and PR~Kumar.
\newblock Exploration through reward biasing: Reward-biased maximum likelihood
  estimation for stochastic multi-armed bandits.
\newblock 2020.

\bibitem[Mandl(1974)]{mandl_74}
P~Mandl.
\newblock Estimation and control in markov chains.
\newblock \emph{Advances in Applied Probability}, pages 40--60, 1974.

\bibitem[Mitzenmacher and Upfal(2017)]{mitzenmacher2017probability}
Michael Mitzenmacher and Eli Upfal.
\newblock \emph{Probability and computing: Randomization and probabilistic
  techniques in algorithms and data analysis}.
\newblock Cambridge university press, 2017.

\bibitem[Osband et~al.(2013)Osband, Russo, and Van~Roy]{osband_13}
Ian Osband, Daniel Russo, and Benjamin Van~Roy.
\newblock (more) efficient reinforcement learning via posterior sampling.
\newblock In \emph{Advances in Neural Information Processing Systems}, pages
  3003--3011, 2013.

\bibitem[Prandini and Campi(2000)]{prandini_00}
Maria Prandini and Marco~C Campi.
\newblock Adaptive lqg control of input-output systems---a cost-biased
  approach.
\newblock \emph{SIAM Journal on Control and Optimization}, 39\penalty0
  (5):\penalty0 1499--1519, 2000.

\bibitem[Puterman(2014)]{puterman2014markov}
Martin~L Puterman.
\newblock \emph{Markov decision processes: discrete stochastic dynamic
  programming}.
\newblock John Wiley \& Sons, 2014.

\bibitem[Rudin(2006)]{rudin2006real}
Walter Rudin.
\newblock \emph{Real and complex analysis}.
\newblock Tata McGraw-hill education, 2006.

\bibitem[Russo and Van~Roy(2018)]{russo_18}
Daniel Russo and Benjamin Van~Roy.
\newblock Learning to optimize via information-directed sampling.
\newblock \emph{Operations Research}, 66\penalty0 (1):\penalty0 230--252, 2018.

\bibitem[Singh et~al.(2020)Singh, Gupta, and Shroff]{singh2020learning}
Rahul Singh, Abhishek Gupta, and Ness~B. Shroff.
\newblock Learning in {M}arkov decision processes under constraints.
\newblock \emph{arXiv preprint arXiv:2002.12435}, 2020.

\bibitem[Srinivas et~al.(2009)Srinivas, Krause, Kakade, and
  Seeger]{srinivas_09}
Niranjan Srinivas, Andreas Krause, Sham~M Kakade, and Matthias Seeger.
\newblock Gaussian process optimization in the bandit setting: No regret and
  experimental design.
\newblock \emph{arXiv preprint arXiv:0912.3995}, 2009.

\bibitem[Srinivas et~al.(2012)Srinivas, Krause, Kakade, and
  Seeger]{srinivas_12}
Niranjan Srinivas, Andreas Krause, Sham~M Kakade, and Matthias~W Seeger.
\newblock Information-theoretic regret bounds for gaussian process optimization
  in the bandit setting.
\newblock \emph{IEEE Transactions on Information Theory}, 58\penalty0
  (5):\penalty0 3250--3265, 2012.

\bibitem[Strens(2000)]{strens_00}
Malcolm Strens.
\newblock A bayesian framework for reinforcement learning.
\newblock In \emph{ICML}, volume 2000, pages 943--950, 2000.

\bibitem[Sutton et~al.(1998)Sutton, Barto, et~al.]{sutton1998introduction}
Richard~S Sutton, Andrew~G Barto, et~al.
\newblock \emph{Introduction to reinforcement learning}, volume 135.
\newblock MIT press Cambridge, 1998.

\end{thebibliography}
